\documentclass[sigconf]{acmart}

\AtBeginDocument{%
  \providecommand\BibTeX{{%
    \normalfont B\kern-0.5em{\scshape i\kern-0.25em b}\kern-0.8em\TeX}}}

\setcopyright{acmcopyright}
\copyrightyear{2023}
\acmYear{2023}
\acmDOI{XXXXXXX.XXXXXXX}

\usepackage{ulem}
\useunder{\uline}{\ul}{}
\usepackage{tcolorbox}
\usepackage{multicol}
\usepackage{multirow}
\usepackage{graphicx}
\usepackage{comment}
\usepackage{dblfloatfix}
\usepackage{enumitem}
\useunder{\uline}{\ul}{}

\newcommand{\mymathhl}[1]{\colorbox{gray!15}{$\displaystyle #1$}}
\newtcolorbox{mybox}[3][]
{
  colframe = #2!25,
  colback  = #2!10,
  coltitle = #2!20!black,  
  title    = {#3},
  #1,
}

\begin{document}

\title{Clenshaw Graph Neural Networks}


\author{Yuhe Guo}
\affiliation{%
  \institution{Renmin University of China}
  \city{Beijing}
  \country{China}
  }
\email{guoyuhe@ruc.edu.cn} 

\author{Zhewei Wei}
\affiliation{%
  \institution{Renmin University of China}
  \city{Beijing}
  \country{China}
  }
\email{zhewei@ruc.edu.cn} 

\renewcommand{\shortauthors}{Trovato and Tobin, et al.}

\begin{abstract}

Graph Convolutional Networks (GCNs), which use a message-passing paradigm with stacked convolution layers, 
are foundational methods for learning graph representations.
{Recent GCN models use various residual connection techniques to alleviate the model degradation problem such as over-smoothing and gradient vanishing.
Existing residual connection techniques, however, fail to 
make extensive use of underlying graph structure as in the graph spectral domain, }
which is critical for obtaining satisfactory results on heterophilic graphs.

In this paper, we introduce ClenshawGCN, a GNN model that employs the Clenshaw Summation Algorithm to enhance the expressiveness of the GCN model.
ClenshawGCN equips the standard GCN model with two straightforward residual modules: the \textit{adaptive initial residual connection} and the \textit{negative second-order residual connection}.
We show that by adding these two residual modules, ClenshawGCN implicitly simulates a polynomial filter under the Chebyshev basis, giving it at least as much expressive power as polynomial spectral GNNs.
In addition, we conduct comprehensive experiments to demonstrate the superiority of our model over spatial and spectral GNN models.



\end{abstract}



\keywords{Graph Neural Networks, Residual Connection, Graph Polynomial Filter}


\settopmatter{printfolios=true}
\maketitle
\section{introduction}

The past few years have witnessed the rise of machine learning on graphs, which considers
relations (edges) between elements (nodes) such as
interactions among molecules \cite{duvenaud2015convolutional, satorras2021n}, 
friendship or hostility between users \cite{fan2019graph, wu2020graph}, and  implicit syntactic or semantic structure in natural language \cite{schlichtkrull2020interpreting, wu2021graph}.

GCN~\cite{kipf2016semi} proposed a message-passing paradigm for Graph Neural Networks that exploits the underlying graph topology by propagating node features iteratively along the edges. 
Along with the message-passing steps, each node receives information from growingly expanding neighborhoods.
Such a \textit{propagation} entangled with non-linear \textit{transformation} 
forms a \textit{graph convolution layer} in GCN. 


\begin{figure*}[htp]
    \includegraphics[width=1.8\columnwidth]{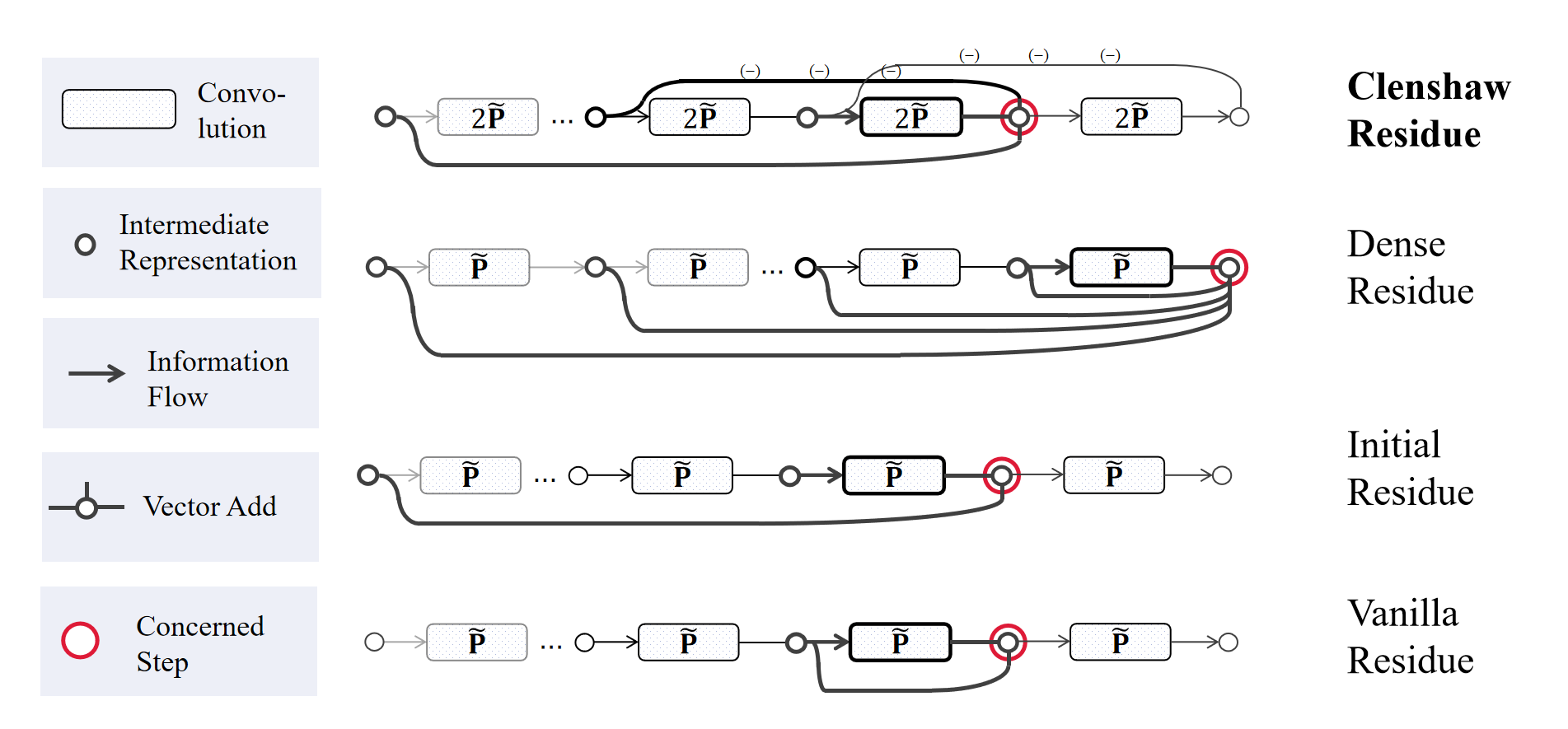}
    \caption{Illustration of different graph residual connections. \textit{Raw Residue} and \textit{Initial Residue} follows a single-line structure, \textit{Dense Residue} connect densely to all the former layers, while our \textit{Clenshaw Residue} achieves strong expressive power by a neat double-line structure. }
    \label{fig:residues}
    \Description{figure description}
\end{figure*}

Following ResNet \cite{He2016residual}, various GNN models use different residual connection techniques to overcome the problem of \textit{model degradation}. 
As shown in Figure~\ref{fig:residues}, we broadly classify the graph residual connections into three types:

\begin{itemize}[leftmargin = 4mm]
    \item \textbf{Raw Residual Connections} are 
    directly transplanted to GCNs in early attempts \cite{kipf2016semi}, 
    but they lose effectiveness when models come deeper\cite{keyulu2018jknet, chen2020gcnii}.   
    The reason behind this was gradually
    clarified with the discussion of {\ul \textit{over-smoothing}} \cite{li2018deeper,rong2020dropedge}
    .  
    \item  
    \textbf{Initial Residual Connections} \cite{klicpera2019appnp, chen2020gcnii} are intuited by \textit{personalized pagerank}, 
    following the perspective of viewing GCNs with raw residues as \textit{lazy random walk} \cite{keyulu2018jknet, wang2019improving}, 
    which converges to the \textit{stationary vector}. 
    Equipped with initial residues,   
    the final representation implicitly leverages features of all fused levels. 
    However, despite its ability to eliminate over-smoothing, 
    models with
    initial residual connections in fact employees the \textit{\underline{homophily assumption}} \cite{mcpherson2001birds},  
    that is, the assumption that connected nodes tend to share similar features and labels, 
    which might be unsuitable for heterophilic graphs~\cite{Pei2020GeomGCN, zheng2022hetesurvey}. 
    Some variants are also explored~\cite{Liu2021Air, Zhang2022hinder}, 
    but they still fall into the scope of homophily. 
    \item 
    There are also \textbf{Dense Residual Connections}, which believe that,    
    by exhibiting feature maps after different times of convolutions and combining them selectively, 
    it is helpful for making extensive use of multi-scale information
    \cite{abu2019mixhop, keyulu2018jknet}.   
    The connection pattern in DenseNet \cite{iandola2014densenet} is also exploited 
    in ~\cite{Li2019deep}, 
    where they connect each layer to all former layers, incurs unaffordable space consumption.
\end{itemize}

\subsubsection*{\ul Residual Connections and Polynomial Filters}
On the other hand, 
the idea of making extensive use of multi-scale representations is closely related to  
\textit{Spectral Graph Neural Networks}, 
which motivates us to {inject the characteristics of spectral GNNs into our model by residual connections}. 
Spectral GNNs {gain their power} 
from utilizing atomic \textit{structural components} decomposed from the underlying graph. 
Spectral GNNs project features onto these components, 
and modulate the importance of the components by a \textit{filter}. 
When this filter is a $K$-order \textit{polynomial} function on the graph's Laplacian spectrum, 
explicit decompositions of matrices are avoided, 
replaced by additions and subtractions of localized propagation results 
within different neighborhood radii, 
building a connection with `residual connection'.


Often, the problem of polynomial filtering is reduced to 
learning a proper polynomial function in order $K$. 
The expressive power in terms of representing \textbf{arbitrary polynomial filters} has been discussed in 
former graph residual connection works, especially dense residues \cite{abu2019mixhop, chien2021gprgnn}.  
However, 
motivated from the spatial view of message passing, 
these works lose an important part of spectral graph learning: the employee of \textbf{polynomial bases}. 
Polynomial bases are of considerable significance when learning filtering functions.  
Until now, different polynomial bases are utilized, including 
Chebyshev Basis \cite{Defferrard2016cheb, he2022chebii}, 
Bernstein Basis \cite{He2021bern}
etc.\
See the example of Runge Phenomenon in \cite{he2022chebii} for an illustration of its importance.

\subsubsection*{\ul Contributions}
This paper focuses on the design of graph residual connections that exploit the expressive power of spectral polynomial filters.
As we have demonstrated, 
raw residual connections and initial residual connections are primarily concerned with mitigating \textit{over-smoothing}, and they commonly rely on the \textit{homophily assumption}.
Some dense residual connections do increase the expressive capacity of a GNN, with analyses aligning the spatial process with a potentially arbitrary spectral polynomial filter.
Unfortunately, these works do not take the significance of different polynomial bases into account and are therefore incapable of learning a more effective spectral node representation.

In this paper, we introduce ClenshawGCN, a GNN model that
employs the Clenshaw Summation Algorithm to {enhance the expressiveness of the GCN model}. 
More specifically, we list our contributions as follows:
\begin{itemize}[leftmargin=4mm]
    \item \textbf{Powerful residual connection submodules}. We propose ClenshawGCN, a message-passing GNN which borrows spectral power by adding two simple residual connection modules to each convolution layer: 
    an \textit{adaptive initial residue} and a \textit{negative second order residue}; 
    \item \textbf{Expressive power in terms of polynomial filters} We show that a $K$-order ClenshawGCN simulates \textit{any} $K$-order polynomial function based on the \textit{Chebyshev basis of the Second Kind}. 
    More specifically, 
    the \textit{adaptive initial residue} allows for 
    the flexibility of coefficients,
    while the \textit{negative second order residue} allows for the implicit use of 
    Chebyshev basis. We prove this by the Clenshaw Summation Algorithm for Chebyshev basis (the Second Kind). 
    \item \textbf{Outstanding performances}. We compare ClenshawGCN with both spatial and spectral models.  
    Extensive empirical studies demonstrate that ClenshawGCN outperforms 
    both spatial and spectral models on various datasets.
\end{itemize}


\section{Preliminaries}
\subsection{Notations}
We consider a simple graph $\mathcal{G}=(\mathcal{V}, \mathcal{E}, \mathbf{A})$, 
where $\mathcal{V} = \{1, 2, \cdots, n \}$ is a finite node-set with $|\mathcal{V}|=n$, 
$\mathcal{E}$ is an edge-set with  $|\mathcal{E}|=m$,
and $\mathbf{A}$ is an unnormalized adjacency matrix.
$\mathcal{L} = \mathbf{D} - \mathbf{A}$ denotes $\mathcal{G}'s$ unnormalized graph Laplacian,
where $\mathbf{D}=\text{diag}\{d_{1}, \cdots, d_{n}\}$
is the degree matrix with $d_i = \sum_{j}{\mathbf{A}_{ij}}$. 

Following GCN, we add a self-connecting edge is to each node,
and conduct symmetric normalization on $\mathcal{L}$ and $\mathbf{A}$.
The resulted \textit{self-looped symmetric-normalized} 
adjacency matrix and Laplacian are denoted as 
$\tilde{\mathbf{P}} = (\mathbf{D}+\mathbf{I})^{-1/2} (\mathbf{A}+\mathbf{I}) (\mathbf{D}+\mathbf{I})^{-1/2}$ and $\tilde{\mathcal{L}}
= \mathbf{I} - \tilde{\mathbf{P}}
$, respectively. 


Further, we attach each node with an $f$-dimensional raw feature and denote the feature matrix as $\mathbf{X} \in \mathbb{R}^{n \times f}$. 
Based on the the topology of underlying graph, 
GNNs enhance the raw node features to better representations for downstream tasks, 
such as node classification or link prediction.

\subsection{Spatial Background}
\subsubsection*{\ul Layer-wise Message Passing Architecture}
From a spatial view, 
the main body of a Graph Neural Network is a stack of \textit{convolution} layers, 
who broadcast and aggregate feature information along the edges. 
Such a graph neural network is also called a Message Passing Neural Network (MPNN).
To be concrete, 
we consider an MPNN with $K$ graph convolution layers 
and denote the nodes' representations of the $\ell$-th layer as $\mathbf{H}^{(\ell)}$. 
$\mathbf{H}^{(\ell)}$ is constructed 
based on $\mathbf{H}^{(\ell-1)}$
by an \textit{propagation} and  possibly a \textit{transformation} operation. 
For example, in Vanilla GCN~\cite{kipf2016semi}, 
the convolution layer is defined as 
\begin{align}
    \label{eq:gcn_conv}
    \mathbf{H}^{(\ell)} = \sigma( \tilde{\mathbf{P}}\mathbf{H}^{(\ell-1)} \mathbf{W}^{(\ell)}) ),
\end{align}
whose $propagation$ operator is $\mymathhl{f: {\mathbf{H}}^{(\ell-1)} \to \tilde{\mathbf{P}}\mathbf{H}^{(\ell-1)}}$ 
and \textit{transformation} operator is $\mymathhl{f: \tilde{\mathbf{P}}\mathbf{H}^{(\ell-1)} \to 
\sigma( \tilde{\mathbf{P}}\mathbf{H}^{(\ell-1)} \mathbf{W}^{(\ell)}) )}$. 

Extra transformations, or probably combinations,  
are applied before and after the stack of 
$K$ convolution layers 
to form a map from  $\mathbf{X}$ to the {final output} $\hat{\mathbf{Y}}$, 
which is determined by the downstream task, 
\textit{e.g.} 
$\mathbf{H}^{(0)}=\text{MLP}(\mathbf{X}; \mathbf{W}^{(0)})$ 
and $\hat{\mathbf{Y}}=\text{SoftMax}(\text{MLP}(\mathbf{H}^{(K)}; \mathbf{W}^{(K+1)}))$ 
for node classification tasks. 
In this paper, a slight difference in notation lies in that, 
we denote the input of convolution layers to be  $\mathbf{H}^{*}$, instead of $\mathbf{H}^{(0)}$,  
and introduce notations $\mathbf{H}^{(-1)}$ and $\mathbf{H}^{(-2)}$ as zero matrices
for simplicity of later representation.


\subsubsection*{\ul Entangled and Disentangled Architectures}
The motivations for \textit{propagation} and \textit{transformation} 
in a convolution layer differ. 
\textit{Propagations} are related to
graph topology, 
analyzed as an analog of 
walks~\cite{keyulu2018jknet, wang2019improving, klicpera2019appnp, chen2020gcnii}, 
diffusion processes~\cite{zhao2021adaptive, klicpera2019diffusion, chamberlain2021grand}, 
etc.\ , 
while the 
\textit{entangling} of \textit{transformations} between \textit{propagations} 
follows behind the convention of deep learning. 
A GNN is classified under a disentangled architecture 
if the \textit{transformations} are disentangled from \textit{propagations}, such as APPNP~\cite{klicpera2019appnp} and GPRGNN~\cite{chien2021gprgnn}. 
Though it is raised in ~\cite{Zhang2022hinder} that 
the entangled architecture tends to cause model degradations, 
we observe that GCNII~\cite{chen2020gcnii} under the entangled architecture does not suffer from this problem. 
So we follow the use of entangled transformations as in Vanilla GCN, 
and leverage the identity mapping of weight matrices as in GCNII.

\subsection{Spectral Background}
\subsubsection*{\ul Spectral Definition of Convolution}
Graph spectral domain leverages the geometric structure of underlying graphs in another way~\cite{Shuman2013}.

Conduct eigen-decomposition on $\tilde{\mathcal{L}}$, 
i.e. $\tilde{\mathcal{L}} = \mathbf{U} \mathbf{\Lambda} \mathbf{U}^T$, 
the spectrum $\mathbf{\Lambda}$  
$= \text{diag} \{ \lambda_1, \cdots, \lambda_n \} $ 
is in non-decreasing order.  
Since $\tilde{\mathcal{L}}$ is real-symmetric, 
elements in $\mathbf{\Lambda}$ 
are real, 
and $\mathbf{U}$ is a complete set of $n$ orthonormal eigenvectors, 
which is used as a basis of \textit{frequency components} 
analogously to \textit{classic Fourier transform}
. 

Now consider a column in $\mathbf{X}$ as a \textit{graph signal} 
scattered on $\mathcal{V}$, 
denoted as $\mathbf{x} \in \mathbb{R}^{n}$. 
\textit{Graph Fourier transform} is defined as 
$\mymathhl{\hat{\mathbf{x}} := \langle \mathbf{U}, \mathbf{x}\rangle = \mathbf{U^T x}}$, 
which projects graph signal $\mathbf{x}$ to the \textit{frequency responses}  of basis components  $\hat{\mathbf{x}}$. 
It is then followed by {\textit{{modulation}}}, 
which can be presented as
$\mymathhl{\mathbf{\hat{x}}^{*}:= g_{\mathbf{\theta}}\mathbf{\hat{x}} = \text{diag}\{\theta_1, \cdots \theta_n \} \mathbf{\hat{x}}} $. 
After modulation, 
$\textit{inverse Fourier transform}$: $\mymathhl{\mathbf{x}^{*}:= \mathbf{U}\mathbf{{\hat{x}}^{*}}}$
transform $ \mathbf{{\hat{x}}^{*}} $ back to the spatial domain.
The three operations form a spectral definition of \textit{convolution}:
\begin{align}
    \label{eq:graph_fourier}
    g_{\theta} \star \mathbf{x} = \mathbf{U}g_{\theta}\mathbf{U}^T\mathbf{x}, 
\end{align}
which is also called \textit{spectral filtering}. 
Specifically, when $\theta_i = 1-\lambda_i$, 
$\mathbf{U}g_{\theta}\mathbf{U}^T\mathbf{x} \equiv \tilde{\mathbf{P}}\mathbf{x}$, 
giving an spectral explanation of GCN's convolution in Equation~\eqref{eq:gcn_conv}.

\subsubsection*{\ul Polynomial Filtering}
\label{sec:poly_filter}
The calculation of $\mathbf{U}$ in Equation~\eqref{eq:graph_fourier} is of prohibitively expensive. 
To avoid explict eigen-decomposition of 
${\mathbf{U}}$, 
$g_{\theta}$ is often defined as a polynomial function of a frequency component's corresponding eigenvalue parameterized by $\theta$, 
that is, 
$$
\mathbf{{\hat{x}}^{*}}_i = g_{\theta}(\lambda_i)\langle \mathbf{U}_i, \mathbf{x} \rangle.
$$
The spectral filtering process then becomes
\begin{align}
    \label{eq:poly_filter}
    g_{\theta} \star \mathbf{x} = \mathbf{U}g_{\theta}(\mathbf{\Lambda})\mathbf{U}^T\mathbf{x} 
                \equiv  g_{\theta}(\tilde{\mathcal{L}})\mathbf{x}, 
\end{align}
where $g_{\theta}(\tilde{\mathcal{L}})\mathbf{x}$ elimates eigen-decomposition 
and can be calculated in a localized way in $\mathcal{O}(|\mathcal{E}|)$
~\cite{Defferrard2016cheb, kipf2016semi}.

\begin{tcolorbox}[boxrule=0.2pt,height=30mm,boxrule=0.pt,valign=center,colback=blue!4!white]
\begin{definition}[Polynomial Filters]
    \label{def:Polynomial Filters}
    Consider a graph whose Laplacian matrix is $\mathcal{\tilde{L}}$ and 
    use the set of orthonormal eigenvectors of $\mathcal{\tilde{L}}$ 
    as the frequency basis, 
    a \textbf{polynomial filter} is a process that scales each frequency component of the input signal by $\mymathhl{g_{\theta}(\lambda)}$, 
    where $g_{\theta}$ is a polynomial function and $\lambda$ is the corresponding 
    eigenvalue of the frequency component.
\end{definition}
\end{tcolorbox}

Equivalently, we can define the filtering function on 
the spectrum of $\mathbf{\tilde{P}}$, instead of $\tilde{\mathcal{L}}$. 
Since $\mathbf{\tilde{P}} = \mathbf{I} - \tilde{\mathcal{L}}$, $\mathbf{\tilde{P}}$ and $\tilde{\mathcal{L}}$ share the same set of orthonormal eigenvectors $\mathbf{U}$, 
and the spectrum of  $\mathbf{\tilde{P}}$, denoted as $\mathbf{M} = \{ \mu_1, \cdots, \mu_n\}$, satisfies 
$\mu_i = 1- \lambda_i (i=1,\cdots, n)$ . 
Thus, the filtering function can be defined as $h_{\theta}$, 
where 
\begin{equation}
    \label{eq:poly_filter_definition}
    \mymathhl{h_{\theta}(\mu) \equiv g_{\theta}(1-\mu)}.
\end{equation}
In Section~\ref{sec:method}, 
for the brevity of presentation, 
we will use this equivalent definition.





\subsection{Polynomial Approximation and Chebyshev Polynomials}
\subsubsection*{\ul Polynomial Approximation}
Following 
the idea of polynomial filtering (Equation~\eqref{eq:poly_filter}), 
the problem 
then becomes the approximation of polynomial $g_{\theta}$. 
A line of work approximates $g_{\theta}$ by 
some truncated polynomial basis $\{\phi_i(x)\}_{i=0}^{i=K}$
up to the $K$-th order, 
\textit{i.e.}  
$$
    g_{\theta}(x) = \sum_{k=0}^{K} \theta_i \phi_i(x),
$$
where
$\vec{\theta} = \left[ \theta_0, \cdots, \theta_K \right]\in \mathbb{R}^{K+1}$ is the coefficients.   
In the field of polynomial filtering and spectral GNNs, 
different bases have been explored for $\{\phi_i(x)\}_{i=0}^{i=K}$, 
including Chebyshev basis~\cite{Defferrard2016cheb}, 
Bernstein basis~\cite{He2021bern}, 
Jacobi basis~\cite{Wang2022jacovi}, {etc}.

\subsubsection*{\ul Chebyshev Polynomials}
Chebyshev basis has been explored since early attempts for the approximation of $g_{\theta}$~\cite{Defferrard2016cheb}. 
Besides Chebyshev polynomials of the first kind ( $\{T_i(x)\}_{i=0}^{\infty}$ ), 
the second kind ( $\{U_i(x)\}_{i=0}^{\infty}$ ) is also wildly used. 

Both $\{T_i(x)\}_{i=0}^{\infty}$ and $\{U_i(x)\}_{i=0}^{\infty}$ 
can be generated by a \textit{recurrence relation}: 
\begin{align}
        \label{eq:Urec}
        &T_{0}(x)  = 1, \quad
        T_{1}(x)  = x, \nonumber \\ 
        &T_{n}(x) = 2xT_{n-1}(x) - T_{n-2}(x). \quad(n=2, 3, \cdots)  
        \nonumber  \\ \nonumber \\
        &U_{-1}(x)  = 0, \quad
        U_{0}(x) = 1, 
        U_{1}(x)  = 2x,\nonumber   \\
        &U_{n}(x) = 2xU_{n-1}(x) - U_{n-2}(x). \quad(n=1, 2, \cdots)
\end{align}



The recurrence relation is used in ChebyNet~\cite{Defferrard2016cheb} 
for accelarating the computing of polynomial filtering.
Note that we start the second kind from $U_{-1}$, which will be used in later proof in Section~\ref{sec:neg_residue}.

\subsection{Residual Network Structures}
We have discussed some graph residual connections in Introduction . 
In this section, we list some model in detail for illustration of residual connections.

\textbf{GCNII} 
equips 
the vanilla GCN convolution with two techniques: initial residue and identity mapping:   
\begin{equation}
        \label{eqn:gcnii_analysis}
        \hspace{-0.7mm} \mathbf{H}^{(\ell)} \hspace{-0.7mm}= \hspace{-0.7mm}
      \sigma  \hspace{-0.7mm}\left(  \hspace{-0.7mm}\left( \hspace{-0.5mm}  (1  \hspace{-0.7mm}-  \hspace{-0.7mm}\alpha)\tilde{\mathbf{P}}
            \mathbf{H}^{(\ell-1)}  \hspace{-0.7mm} +  \hspace{-0.7mm}
            \alpha\mathbf{H}^{*}  \hspace{-0.7mm}\right)  \hspace{-0.7mm}
          \left(  \hspace{-0.5mm}   (1  \hspace{-0.7mm} -  \hspace{-0.7mm}\beta_\ell) \mathbf{I}_n \hspace{-0.7mm} +
            \hspace{-0.7mm} \beta_\ell \mathbf{W}^{(\ell)}  \hspace{-0.7mm}\right)  \hspace{-0.7mm}\right).
        \end{equation}

Ignoring non-linear transformation, GCNII 
iteratively solves the optimization problem:
\begin{equation}
    \label{eq:opt_initial}
    \underset{\mathbf{H}}{\arg \min} \quad  
    {\alpha}\left\|\mathbf{H}-\mathbf{H}^{*}\right\|_F^2+ (1-\alpha) \operatorname{tr}\left(\mathbf{H}^{\top} \tilde{\mathcal{L}} \mathbf{H}\right).   
\end{equation}

The optimization goal reveals the underlying 
\textit{homophily assumption}, 
where initial residual connection is only making a \textit{compromise} 
between Laplacian smoothing and keeping identity.

\textbf{AirGNN}~\cite{Liu2021Air} proposes an extension for initial residue, 
where the first term of optimization problem in Equation~\eqref{eq:opt_initial} 
is replaced by the $\ell_{21}$ norm. 
By solving the optimization problem, AirGNN adaptively choose $\alpha$. 
However, limited by the optimization goal, 
AirGNN still falls into the homophility assumption. 

\textbf{JKNet}~\cite{keyulu2018jknet} 
uses dense residual connection at the last layer  
and combine all the intermediate representations nodewisely by different ways, 
including LSTM, Max-Pooling and so on. 

\textbf{MixHop}~\cite{abu2019mixhop} 
concats feature maps of several hops at each layer, represented as
\begin{equation*}
\mathbf{H}^{(\ell+1)} =\underset{j \in K}{\Bigg\Vert} \sigma\left(\tilde{\mathbf{P}}^j \mathbf{H}^{(\ell)} \mathbf{W}_j^{(\ell)}\right),  
\end{equation*}
which can be considered as staking several dense graph residual networks.

\section{Method}
\label{sec:method}

\subsection{Clenshaw Convolution}


We formulate the $\ell$-th layer's representation of ClenshawGCN as 
 \begin{equation}
    \label{eq:clenshaw_conv}
            \hspace{-0.7mm} \mathbf{H}^{(\ell)} \hspace{-0.7mm}= \hspace{-0.7mm}
           \sigma  \hspace{-0.7mm}
           \left(  \hspace{-0.7mm}
                \left(
                    2\mathbf{\tilde{P}}\mathbf{H}^{(\ell-1)} 
                    \mymathhl{-\mathbf{H}^{(\ell-2)}}
                    + 
                    \mymathhl{\alpha_{\ell}\mathbf{H}^{*}}
                \right)
                \hspace{-0.7mm}
              \left(  \hspace{-0.5mm}   (1  \hspace{-0.7mm} -  \hspace{-0.7mm}\beta_\ell) \mathbf{I}_n \hspace{-0.7mm} +
                \hspace{-0.7mm} \beta_\ell \mathbf{W}^{(\ell)}  \hspace{-0.7mm}\right)  \hspace{-0.7mm}\right),
        \end{equation}
where $\ell = 0, 1, \ldots , K$, $\mathbf{H}^{(-2)}=\mathbf{H}^{(-1)}=\mathbf{O}$, 
$\mathbf{H}^{*} = \text{MLP}(\mathbf{X}; \mathbf{W^{*}})$.

Note that 
for \textit{transformation}, we use identity mapping 
with $\beta_{\ell} = \log (\frac{\lambda}{\ell} + 1) \approx \lambda / \ell$
following GCNII~\cite{chen2020gcnii}. 
Comparing to GCNII, we include two simple yet effective residual connections: 
\textbf{Adaptive Initial Residue} and \textbf{Negative Second Order Residue}:

 \begin{equation*}
        2\mathbf{\tilde{P}}\mathbf{H}^{(\ell-1)} 
        \underbrace{\mymathhl{-\mathbf{H}^{(\ell-2)}}}_{\substack{\text{
                    \textit{Negative} \underline{Second}} \\ \text{\underline{Order} Residue}}}
        + 
        \underbrace{\mymathhl{\alpha_{\ell}\mathbf{H}^{*}}
        .
        }_{ 
        \substack{
            \text{\textit{Adaptive}} \\ \text{\underline{Initial} Residue}
            }
        }
\end{equation*}


In the remaining part of this section, we will illustrate the role and mechanism of these two residual modules. 
In summary, 
\textit{adaptive initial residue} enables the simulating of any $K$-order polynomial filter, 
while \textit{negative second order residue}, 
motivated by the leveraging of `differencing relations', 
simulates the use of Chebyshev basis in the approximation of filtering functions.
For both two parts, 
we will give an intuitive analysis
with followed by a proofs.


\subsection{Adaptive Initial Residue}

The role of {\ul \textit{adaptive}} {\ul \textit{initial residue}} is to enable 
the expressive power of \textit{any} $K$-order polynomial filter.
To illustrate this, 
we will first consider an \textbf{incomplete} version of the ClenshawGCN termed HornerGCN. 

\subsubsection*{\underline{Intuition}}
We start with an analysis of GCNII. 
To simplify the analysis, we consider 
$(1 - \beta_\ell) \mathbf{I}_n  +\beta_\ell \mathbf{W}^{(\ell)} $
as $\mathbf{I}$, 
and take $\textrm{relu}(x) = x$. 
At this point, the iteration~\eqref{eqn:gcnii_analysis} is simplified as
\begin{equation}
    \label{eqn:gcnii_analysis_simple}
    \hspace{-0.7mm} \mathbf{H}^{(\ell)} =
        (1  \hspace{-0.7mm}-  \hspace{-0.7mm}\alpha)\tilde{\mathbf{P}}
        \mathbf{H}^{(\ell-1)}  \hspace{-0.7mm} +  \hspace{-0.7mm}
        \alpha\mathbf{H}^{*}  .
\end{equation}

Consider a GCNII model of order $K$\footnote{For simplicity, we term a model 
with at most $K$ propagations as of order $K$}, 
by expanding \eqref{eqn:gcnii_analysis_simple},  
we obtain that 

\begin{equation}
    \label{eqn:gcnii_analysis_unfolded}
    \mathbf{H}^{(K)}= \sum_{\ell=0}^{K} \hat{\alpha}_{\ell} \tilde{\mathbf{P}}^{\ell} \mathbf{H}^{*}, 
\end{equation}
where 
$$
\hat{\alpha}_{\ell} =
\left\{\begin{matrix}
    \alpha(1-\alpha)^{\ell}, & \ell<  K, \\ 
     (1-\alpha)^{K}, & \ell = K.
\end{matrix}\right.
$$

It can be seen that the iterative process of GCNII 
implicitly leverages the representations of different layers 
with \textit{fixed} and \textit{positive} coefficients.  
However, we want \textit{flexible} exploitation of different diffusion layers, \cite{keyulu2018jknet,zhao2021adaptive}. 
More importantly, we need to have \textit{negative} weights 
in cases of heterophilic graphs \cite{chien2021gprgnn}.
To  this end, we can simply replace the fixed $\alpha$ in 
\eqref{eqn:gcnii_analysis}
with \textit{learnable} ones. 
In this way, we obtain
the form of HornerGCN
as follows,  

\begin{equation}
        \label{eq:horner_conv}
            \hspace{-0.7mm} \mathbf{H}^{(\ell)} \hspace{-0.7mm}= \hspace{-0.7mm}
           \sigma  \hspace{-0.7mm}
           \left(  \hspace{-0.7mm}
                \left(
                    \mathbf{\tilde{P}}\mathbf{H}^{(\ell-1)} 
                    +
                    \mymathhl{\alpha_{\ell}\mathbf{H}^{*}}
                \right)
                \hspace{-0.7mm}
              \left(  \hspace{-0.5mm}   (1  \hspace{-0.7mm} -  \hspace{-0.7mm}\beta_\ell) \mathbf{I}_n \hspace{-0.7mm} +
                \hspace{-0.7mm} \beta_\ell \mathbf{W}^{(\ell)}  \hspace{-0.7mm}\right)  \hspace{-0.7mm}\right),
\end{equation}
where $\ell = 0, 1, \ldots, K$, $\mathbf{H}^{(-1)}=\mathbf{O}$, 
$\mathbf{H}^{*} = \text{MLP}(\mathbf{X}; \mathbf{W^{*}})$.

\subsubsection*{\underline{Spectral Nature}}
With the definition of \textit{polynomial filters} given in 
Equation~\eqref{eq:poly_filter_definition} and Definition~\ref{def:Polynomial Filters}, 
we will prove the Theorem below:
\begin{tcolorbox}[boxrule=0.2pt,height=40mm,boxrule=0.pt,valign=center,colback=blue!4!white]
    \begin{theorem}
    \label{thm:horner}
    A $K$-order HornerGCN 
    defined in Equation~\eqref{eq:horner_conv}, 
    when consider 
    $(1 - \beta_\ell) \mathbf{I}_n  +\beta_\ell \mathbf{W}^{(\ell)} $ for each $\ell$
    as 
    $\mathbf{I}$, 
    and $\textrm{relu}(x)$ as $x$, 
    simulates a \textit{polynomial filter} on the monomial basis: 
    
    $$ 
    h(\mu) = \sum_{\ell=0}^{K}{\alpha}_{K-\ell} \mu^{\ell},
    $$ 
    
    where $ \{\alpha_{\ell}\}_{\ell=0}^{K}$ is the set of 
    \textit{initial residue coefficients}.   
    \end{theorem}
\end{tcolorbox}

\subsubsection*{\underline{Horner's Method}}
To prove Theorem~\ref{thm:horner}, 
we first briefly introduce Horner's Method \cite{horner1819xxi}.
Given $p(x) = \sum_{i=0}^{n} a_i x^i = a_0 + a_1 x + \cdots a_n x^n$, 
Horner's Method is a classic method for evaluating 
$p(x_0)$, 
by  
viewing  $p(x)$ as the following form:
\begin{align}
    \label{eq:horner_form}
    p(x)=a_0+x\left(a_1+x\left(a_2+x\left(a_3+\cdots+x\left(a_{n-1}+x a_n\right) \cdots\right)\right)\right) .
\end{align}

Thus, Horner's method defines a recursive method
for evaluating $p(x_0)$: 

\begin{align}
    \label{eq:horner_recursive}
    b_n &:=a_n,  \nonumber \\
    b_{n-1} &:=a_{n-1}+b_n x_0, \nonumber  \\
    &\cdots  \nonumber  \\
    b_0 &:=a_0+b_1 x_0,  \nonumber  \\
    p(x_0) &:= b_0.
\end{align}


\subsubsection*{\underline{Proof of Spectral Expressiveness}}
Note that the form of Horner's recursive 
are parallel with the recursive of stacked Horner convolutions by ignoring the non-linear transformations.
Thus, by unfolding the nested expression of Horner convolutions \eqref{eq:horner_conv}, 
we get the output of the last layer closely matched with the form of Equation~\eqref{eq:horner_form} 
: 
\begin{tcolorbox}[boxrule=0.2pt,height=45mm,boxrule=0.pt,valign=center,colback=blue!4!white]
    \begin{align}
            \label{eqn:horner_expansion_layerwise}
            \hspace{-0.7mm} 
            \mathbf{H}^{(0)} &= \alpha_{0}\mathbf{H}^{*},  \nonumber \\
            \mathbf{H}^{(1)} &= \tilde{\mathbf{P}} \left( \alpha_{0}\mathbf{H}^{*} \right)  + \alpha_{1}\mathbf{H}^{*}, 
            \nonumber \\
            \cdots \nonumber \\ 
            \mathbf{H}^{(K)} &=  \tilde{\mathbf{P}} \left(  \cdots \left(
                            \tilde{\mathbf{P}} \left( 
                                            \tilde{\mathbf{P}} \left( \alpha_{0}\mathbf{H}^{*} \right)  + \alpha_{1}\mathbf{H}^{*} 
                                            \right) + \alpha_{2}\mathbf{H}^{*}
                            \right) \cdots \right) + \alpha_{K}\mathbf{H}^{*}
                            \nonumber \\
                            &= \alpha_K\mathbf{H}^{*} + \alpha_{K-1} \tilde{\mathbf{P}} \mathbf{H}^{*} + \cdots + \alpha_{0} \tilde{\mathbf{P}}^{K} 
                            \nonumber \\
                            &= \sum_{\ell=0}^{K} \alpha_{K-\ell} \tilde{\mathbf{P}}^{\ell} \mathbf{H}^{*}. 
    \end{align}
\end{tcolorbox}


So, the final representation 
\begin{align*}
    \mathbf{H}^{(K)} &= 
\sum_{\ell=0}^{K} \alpha_{K-\ell} \tilde{\mathbf{P}}^{\ell} \mathbf{H}^{*}
= \mathbf{U} 
    \left(
        \sum_{\ell=0}^{K}{\alpha}_{K-\ell} \mathbf{M}^{\ell} 
    \right)
    \mathbf{U}^T \mathbf{H}^{*} ,
\end{align*}
corresponds to the result of 
a polynomial filter $h$ on the spectrum of $\mathbf{\tilde{P}}$, where 
$$
    h(\mu) = 
    \sum_{\ell=0}^{K}{\alpha}_{K-\ell} \mu^{\ell}.
$$



\subsection{Negative Second Order Residue}
\label{sec:neg_residue}

The role of {\ul \textit{negative second order residue}} is to enable the leverage of Chebyshev basis. 
Compared with the adaptive initial residue, it is not obvious. 
In this section, we will first give an \textbf{intuitive} motivation 
for Negative Second Order Residue, 
and then reveal the mechanism behind it by Clenshaw Summation Algorithm.

\subsubsection*{\underline{Intuition of Taking Difference}}
There is already some work that has noticed the `subtraction' relationship between progressive levels of representation~\cite{abu2019mixhop, Yang2022difference}. 
For example, one of MixHop's direct challenges towards the traditional GCN model is the inability to represent the \textit{Delta Operator}, \textit{i.e.}, 
$\mathbf{AX}^2\textcolor{red}{-}\mathbf{AX}$, 
which is an important relationship in representing both the concept of social boundaries~\cite{perozzi2018socialBoundary} and the concept of `sharpening' in images~\cite{burt1987laplacian, paris2011local}.

Instead of considering model's ability to represent \textit{Delta Operator} 
in the final output, 
we take a more direct use of these difference relations. 
That is, we add
$
\mymathhl{\mathbf{\tilde{P}}\mathbf{H}^{(\ell-1)} -\mathbf{H}^{(\ell-2)}}
$
to \textit{each} convolution layers of \eqref{eq:horner_conv}, 
and get the final form of Clenshaw Convolution:
$$
\mathbf{\tilde{P}}\mathbf{H}^{(\ell-1)} + \alpha_{\ell}\mathbf{H}^{*}
+ 
\mymathhl{\mathbf{\tilde{P}}\mathbf{H}^{(\ell-1)} -\mathbf{H}^{(\ell-2)}} 
\Longrightarrow 
\eqref{eq:clenshaw_conv}.
$$

A question that perhaps needs to be answered is: why not instead 
insert a more direct difference relation: 
$
\underline{\mathbf{\tilde{P}}\mathbf{H}^{(\ell-1)} -\mathbf{H}^{(\ell-1)}}
$
?
The reason is, 
at this point, the convolution would become
$$
\mathbf{\tilde{P}}\mathbf{H}^{(\ell-1)} + \alpha_{\ell}\mathbf{H}^{*}
+ 
\underline{\mathbf{\tilde{P}}\mathbf{H}^{(\ell-1)} -\mathbf{H}^{(\ell-1)}}
= 
\left(2\mathbf{\tilde{P}}-\mathbf{I}\right)\mathbf{H}^{(\ell-1)}
+ \alpha_{\ell}\mathbf{H}^{*},
$$
whose unfolded form, 
by substitute $\mathbf{\tilde{P}}$ to $2\mathbf{\tilde{P}}- \mathbf{I}$ in \eqref{eqn:horner_expansion_layerwise},  
would become $ \sum_{\ell=0}^{K} \alpha_{K-\ell} 
{\left(2\tilde{\mathbf{P}}-\mathbf{I}\right)}^{\ell} 
\mathbf{H}^{*}  $
, 
 which brings limited change.

\subsubsection*{\underline{Spectral Nature}}
Besides the spatial intuition of considering substractions or boundaries, 
we reveal the spectral nature of our \textit{negative residual connection} in this part, that is, it simulates the leverage of chebyshev basis as in spectral polynomials.
With the definition of \textit{polynomial filters} given in 
Equation~\eqref{eq:poly_filter_definition} and Definition~\ref{def:Polynomial Filters}, 
we will prove the theorem below:
\begin{tcolorbox}[boxrule=0.2pt,height=45mm,boxrule=0.pt,valign=center,colback=blue!4!white]
\begin{theorem}
    \label{thm:spatial_to_spectral}
    A $K$-order ClenshawGCN 
    defined in Equation~\eqref{eq:clenshaw_conv}, 
    when consider 
    $(1 - \beta_\ell) \mathbf{I}_n  +\beta_\ell \mathbf{W}^{(\ell)} $ for each $\ell$
    as 
    $\mathbf{I}$, 
    and $\textrm{relu}(x)$ as $x$, 
    simulates a \textit{polynomial filter} 
    $$ 
    h(\mu) = \sum_{\ell=0}^{K}{\alpha}_{K-\ell} U_{\ell}(\mu),
    $$ 
    where $ \{U_{\ell}\}_{\ell=0}^{K}$  is the truncated $K$-order  
    second-kind Chebyshev basis, 
    and $ \{\alpha_{\ell}\}_{\ell=0}^{K}$ is the set of 
    initial residue coefficients.   
    \end{theorem}
\end{tcolorbox}
which can also be expressed as:
\begin{align}
    \label{eq:hK}
    \mathbf{H}^{(K)} 
    = 
    \mathbf{U} 
    h({\mathbf{M}})
    \mathbf{U}^T \mathbf{H}^{*}
    = 
    \mathbf{U} 
    \left(
        \sum_{\ell=0}^{K}{\alpha}_{K-\ell} U_{\ell} (\mathbf{M}) 
    \right)
    \mathbf{U}^T \mathbf{H}^{*}.
\end{align}

\subsubsection*{\underline{Clenshaw Algorithm}}
For the proof of Theorem~\ref{thm:spatial_to_spectral}, 
we will first introduce Clenshaw Summation Algorithm 
as a Corollary. 

\begin{tcolorbox}[boxrule=0.pt,height=60mm,boxrule=0.pt,valign=center,colback=blue!4!white]
   \begin{corollary}[Clenshaw Summation Algorithm for Chebyshev polynomials (Second Kind)]
    \label{cor:clenshaw}
    For the Second Kind of Chebyshev Polynomials, 
    the weighted sum of a finite series of $\{U_{k}(x)\}_{k=0}^{n}$ :
    \[
        S(x) = \sum_{k=0}^{k=n}a_k U_k(x) 
        \] 
    can be computed 
    by a recurrence formula:
    \begin{align}
        \label{eq:clenshaw_recurrence}
        b_{n+2}(x) &:= 0,  \nonumber
        \\
        b_{n+1}(x) &:= 0, \nonumber
        \\
        b_k(x) &:= a_k + 2x b_{k+1}(x) - b_{k+2}(x). \nonumber
        \\ &\quad\quad (k=n, n-1, ..., 0)
    \end{align} 
    Then $S(x) \equiv b_0(x)$.
\end{corollary}
\end{tcolorbox}


Clenshaw Summation can be applied to a wider range of polynomial basis. 
However, specifically for the second kind of Chebyshev, 
we made some slight simplifications\footnote{To be more precise, the simplification 
lies in that Clenshaw Summation procedures for the more general situation 
need an `extra' different step which is not needed in our proof.}, 
so for clarity of discussion, we give the proof here. 

\begin{proof}[Proof for Corollary~\ref{cor:clenshaw}]
Denote 
$$
A =  \begin{bmatrix}
            1 &  &  &  & \\  -2x & 1 &  &  & \\  1& -2x &1  &  & \\ 
             &  & \cdots &  & \\  &  &  1&  -2x & 1 
    \end{bmatrix}, \quad
\vec{u} =   \begin{bmatrix}
                    U_{-1}(x)\\U_{0}(x)\\U_{1}(x)\\\cdots\\U_n(x)
            \end{bmatrix}, \quad
\vec{a} = \begin{bmatrix}
    0\\a_{0}\\a_{1}\\\cdots\\a_n
\end{bmatrix},
$$
with $A \in \mathbb{R}^{(n+2) \times (n+2)}$, $\vec{u} \in \mathbb{R}^{n+2}$, 
then   
\begin{align}
    \label{eq:s}
    S(x) =\vec{a}^T \vec{u}.
\end{align}
Indexing a vector from $-1$, 
we denote $\mathbf{1}_i \in \mathbb{R}^{(n+2)}$ as 
the one-hot vector with the $i$-th element being $1$ ($i$ starts from $-1$). 
Note that, since the recurrence relation for the Chebyshev polynomials~\eqref{eq:Urec},
\begin{align}
    \label{eq:cheb2-3-term}
    A \vec{u} =
    \begin{bmatrix}
    U_{-1}(x) \\ -2xU_{-1}(x) + U_{0}(x) \\ 0 \\ \cdots \\ 0 
    \end{bmatrix}
    =
    \begin{bmatrix}
    0 \\ U_{0}(x) \\ 0 \\ \cdots \\ 0 
    \end{bmatrix}
    =
    \mathbf{1}_0.
\end{align}

Now suppose that there is a vector $\vec{b}  = [ b_{-1}, b_0, \cdots, b_n  ]$ 
satisfying 
\begin{align}
    \label{eq:suppose_b}
    \vec{a}^T = \vec{b}^TA ,
\end{align}
then 
$$
S(x) \stackrel{\eqref{eq:s}}{=} \vec{a}^T \vec{u} 
\stackrel{\eqref{eq:suppose_b}}{=} \vec{b}^TA\vec{u} 
\stackrel{\eqref{eq:cheb2-3-term}}{=} \vec{b} \mathbf{1}_0 = b_0.
$$
On the other hand, 
notice that the recurrence defined in \eqref{eq:clenshaw_recurrence} 
is exactly the Gaussian Elimination process of solving  $\vec{a}^T = \vec{b}^TA $ 
from $b_n$ down to $b_0$, 
which means that $\{b_n, \cdots, b_0\}$ calculated by \eqref{eq:clenshaw_recurrence}  
satisfies \eqref{eq:suppose_b}. Proof for Corollary~\ref{cor:clenshaw} is finished.
\end{proof}

\subsubsection*{\underline{Proof of Spectral Expressiveness}} 
Now we prove Theorem~\ref{thm:spatial_to_spectral} inductively based on Corollary~\ref{cor:clenshaw}.
    
\begin{proof}[Proof of Theorem~\ref{thm:spatial_to_spectral}]
    \textbf{Given}: 
         \begin{equation}
    \label{eq:clenshaw_conv_linear}
             \mathbf{H}^{(\ell)} = 
                    2\mathbf{\tilde{P}}\mathbf{H}^{(\ell-1)} 
                    {-\mathbf{H}^{(\ell-2)}} + {\alpha_{\ell}\mathbf{H}^{*}}, 
    \end{equation}
    $\mathbf{H}^{(-1)}=\mathbf{H}^{(-2)}=\mathbf{0}$.

    \textbf{Induction Hypothesis}: 
    Suppose that when the convolutions have processed to the $\ell$-th layer, 
    $\mathbf{H}^{(\ell-1)}$ and $\mathbf{H}^{(\ell-2)}$ 
    are already proved to be polynomial filtered results of $\mathbf{H}^{*}$,  
    we show that $\mathbf{H}^{(\ell)}$ is also polynomial filtered results of $\mathbf{H}^{*}$.
    
    Further, denote  the polynomial filtering functions
    of generating $\mathbf{H}^{(\ell-2)}$, $\mathbf{H}^{(\ell-1)}$ and $\mathbf{H}^{(\ell)}$ 
    to be  $h^{(\ell-2)}$,  $h^{(\ell-1)}$ and $h^{(\ell)}$.   
    Then $h^{(\ell)}$ satisfies:    
    \begin{align*}
        h^{(\ell)}(\mu) &= \alpha_{\ell} + 2\mu h^{(\ell-1)}(\mu) - h^{(\ell-2)}(\mu)   .
    \end{align*}

\textbf{Base Case}: 
For $\ell=0$, since $\mathbf{H}^{(-2)}=\mathbf{H}^{(-1)}=\mathbf{0}$, 
$ \mathbf{H}^{(0)} = \alpha_0 \mathbf{H}^{*}$, 
the first induction step is established with 
$$
h^{(0)}(\mu) = \alpha_0, 
\quad h^{(-1)}(\mu) = 0, 
\quad  h^{(-2)}(\mu) = 0.
$$

\textbf{Induction Step}: 
Insert
\begin{align*}
    \left\{\begin{matrix}
            \mathbf{H}^{(\ell-1)} 
                &= \mathbf{U} h^{\ell-1}(\mathbf{M})\mathbf{U}^T\mathbf{H}^{*}
            , 
            \\
            \mathbf{H}^{(\ell-2)} 
                &= \mathbf{U} h^{\ell-2}(\mathbf{M})\mathbf{U}^T\mathbf{H}^{*},
            \\
            \mathbf{\tilde{P}} 
                &=  \mathbf{U} \mathbf{M}\mathbf{U}^T 
\end{matrix}\right.
\end{align*}
into Equation~\eqref{eq:clenshaw_conv_linear}, we get 
     \begin{align*}
             \mathbf{H}^{(\ell)} &= 
                    2\underline{\mathbf{U} \mathbf{M}\mathbf{U}^T}
                    \dashuline{\mathbf{U} h^{(\ell-1)}(\mathbf{M})\mathbf{U}^T\mathbf{H}^{*}}
                    - \dashuline{\mathbf{U} h^{(\ell-2)}(\mathbf{M})\mathbf{U}^T\mathbf{H}^{*}}
                    + {\alpha_{\ell}\mathbf{H}^{*}}
                    \\
                    &= \mathbf{U} \left(  
                    \alpha_{\ell} + 2\mathbf{M}h^{(\ell-1)}(\mathbf{M}) - h^{(\ell-2)}(\mathbf{M}) 
                    \right)
                    \mathbf{U}^T\mathbf{H}^{*}.
    \end{align*}
So, $\mathbf{H}^{(\ell)}$ is also a polynomial filtered result of $\mathbf{H}^{*}$, 
with filtering function $h^{(\ell)}$:  
\begin{align}
\label{eq:relation}
h^{(\ell)}(\mu) := \alpha_{\ell} + 2\mu h^{(\ell-1)}(\mu) - h^{(\ell-2)}(\mu)   .
\end{align}

Stack relation\eqref{eq:relation} for $\ell=0,1,\cdots,K$, 
we get: 
\begin{align*}
    h^{(-2)}(\mu) &= 0, \\
    h^{(-1)}(\mu) &= 0, \\
    h^{(\ell)}(\mu) &:= \alpha_{\ell} + 2\mu h^{(\ell-1)}(\mu) - h^{(\ell-2)}(\mu), \quad
    \\ &\quad\quad 
    (k=0, 1, \cdots, K).
\end{align*}
where the progressive access of 
$b^{(\ell)}$ 
is in a \textbf{totally parallel} way with the recurrence \eqref{eq:clenshaw_recurrence} in Clenshaw Summation Algorithm. We soonly get 
$$
\sum_{\ell=0}^{K}{\alpha}_{K-\ell} U_{\ell} (\mu)  \equiv  h^{(K)}(\mu).
$$
Thus, we have finshed the proof of Theorem~\ref{thm:spatial_to_spectral}
.
\end{proof}


\section{Experiments}
In this section, 
we conduct two sets of experiments with the node classification task. 
First, we verify the power of our method by comparing it with both spatial residual methods and powerful spectral models. 
Second, we verify the effectiveness of the two submodules by ablation studies. 

\subsection{Experimental Setup}
\subsubsection*{\underline{Datasets and Splits}}

We use both homophilic graphs and heterophilic graphs in our experiments 
following former works, especially GCN~\cite{kipf2016semi}, Geom-GCN~\cite{Pei2020GeomGCN} and LINKX~\cite{Lim2021large}. 
\begin{itemize}[leftmargin= * ]
    \item \textit{Citation Graphs}.    
        Cora, PubMed, and CiteSeer are citation datasets~\cite{sen2008collective} 
        processed by Planetoid~\cite{yang2016revisiting}. 
        In these graphs, nodes are scientific publications, edges are citation links processed to be bidirectional, 
        and node features are bag-of-words representations of the documents. These graphs show strong homophily. 
    \item \textit{Wikipedia Graphs}. 
        Chameleon dataset and Squirrel dataset are page-page networks on topics in Wikipedia, 
        where nodes are entries, and edges are mutual links. 
    \item \textit{Webpage Graphs}. 
        Texas dataset and Cornell dataset 
        collect web pages from computer science departments of different universities. 
        The nodes in the graphs are web pages of students, projects, courses, staff or faculties~\cite{craven1998learning}, 
        the edges are hyperlinks between them, 
        and node features are the bag-of-words representations of these web pages. 
    \item \textit{Co-occurrence Network}. 
        The Actor network represents the co-occurrence of actors on a Wikipedia page~\cite{tang2009social}. 
        The node features are filtered keywords in the Wikipedia pages. 
        The categorization of the nodes is done by~\cite{Pei2020GeomGCN}.
    \item \textit{Mutual follower Network}. 
        Twitch-Gamers dataset represents the mutual following relationship between accounts on the streaming platform Twitch.  
\end{itemize}

We list the messages of these networks in Table~\ref{tbl:datasets}, where 
$\mathcal{H}(G)$ is the measure of homophily in a graph proposed by Geom-GCN~\cite{Pei2020GeomGCN}. 
Larger $\mathcal{H}(G)$ implies stronger homophily. 

\begin{table}[htp]
    \centering
    \caption{Statistics for the node classification datasets we use. 
    Datasets of different homophily degrees are used.  }
    \resizebox{\columnwidth}{!}{%
    \begin{tabular}{crrcc}
    \toprule
    Dataset       & \#Nodes       & \#Edges         & \#Classes &  $\mathcal{H}(G)$   \\ 
    \midrule
    Cora          & 2,709   & 5,429     & 7 &   .83   \\
    Pubmed        & 19,717  & 44,338    & 3 &   .71   \\
    Citeseer      & 3,327   & 4,732     & 6 &   .79   \\ 
    Squirrel      & 5,201   & 217,073   & 5 &  .22    \\
    Chameleon     & 7,600   & 33,544    & 5 &  .23    \\
    Texas         & 183     & 309       & 5 &  .11    \\
    Cornell       & 183     & 295       & 5 &   .30   \\
    Twitch-Gamers & 168,114 & 6,797,557 & 7 & .55 \\ 
    \bottomrule
    \end{tabular}
    \label{tbl:datasets}
    }
    \end{table}

For all datasets except for the Twitch-gamers dataset, 
we take a 60\%/20\%/20\% train/validation/test split proportion 
following former works~\cite{Pei2020GeomGCN, chien2021gprgnn, He2021bern, he2022chebii}.
We run these datasets twenty times over random splits with random initialization seeds. 
For the Twitch-gamers dataset, 
we use the five random splits given in LINKX~\cite{Lim2021large} with a 50\%/25\%/25\% proportion
to align with reported results. 

\subsubsection*{\underline{ClenshawGCN Setup}}
Before and after the stack of Clenshaw convolution layers, 
two non-linear transformations are made to link with the dimensions of the raw features and final class numbers.     
All the intermediate transformation layers are set with 64 hidden units.  
For the initialization of the adaptive initial residues' coefficients, denoted as 
$\vec{\alpha} = \left[  \alpha_0, \cdots, \alpha_K \right]$, 
we simply set ${\alpha_K}$ to be $1$ and all the other coefficients to be $0$, 
which corresponds to initializing
the polynomial filter to be $\mymathhl{g(\lambda) = 1}$ 
(or equivalently, $\mymathhl{h(\mu) = 1}$ ). 

\subsubsection*{\underline{Hyperparameter Tuning}}
For the optimization process on the training sets, we tune $\vec{\alpha}$ with SGD optimizer with momentum~\cite{sutskever2013momentumImportance} 
and all the other parameters with Adam SGD~\cite{kingma2014adam}. 
We use early stopping with a patience of 300 epochs.

For the search space of hyperparameters, 
below is the search space of hyperparameters:
\begin{itemize}[leftmargin=7mm]
    \item Orders of convolutions: $K \in \left\{ 8, 12, \cdots, 32 \right\}$;
    \item Learning rates: $ \left\{ 0.001, 0.005,  0.1,  0.2,  0.3, 0.4, 0.5   \right\}$; 
    \item Weight decays: $\left\{ 1\mathrm{e}{-8}, 1\mathrm{e}{-7}, \cdots,  1\mathrm{e}{-3}   \right\}$;
    \item Dropout rates: $\left\{ 0,  0.1,  \cdots,  0.7   \right\}$.
\end{itemize}

We tune all the hyperparameters on the validation sets. 
To accelerate hyperparameter searching, we use Optuna~\cite{akiba2019optuna} and run 100 completed  
trials \footnote{
In Optuna, a \textit{trial} means a run with hyperparameter combination; 
the term `complete' refers to that, 
some trials of bad expectations would be pruned before completion.
}
for each dataset.

\subsection{Comparing ClenshawGCN with Other Graph Residual Connections} 

\begin{table*}[htp]
\centering
\small
\caption{
Comparison with other models equipped by different kinds of residual connections.
Mean classification accuracies {(± standard derivation)}
of random splits are displayed. Besides the ClenshawGCN,
all the results are taken directly from \cite{Luan2021snowball} and \cite{Lim2021large}.
For the Twitch-Gamer dataset, we use 5 fixed 50\%/25\%/25\% splits
given in \cite{Lim2021large} to align with the reported results. 
For all the other datasets, 20 random 60\%/20\%/20\%  splits were used. 
(M) denotes some hyperparameter settings run out of memory \cite{Lim2021large}.
}
\resizebox{1.98\columnwidth}{!}{%
\begin{tabular}{@{}llllllllll@{}}
\toprule
  Datasets &Chameleon &Squirrel &Actor &Texas &Cornell & Cora & Citeseer & Pubmed & Twitch-gamer \\
  $|\mathcal{V}|$ &2,277 & 5,201 & 7,600 & 183  & 183 & 2,708 & 3,327 & 19,717 & 168,114\\
\midrule
{MLP} &
  46.59±1.84 &
  31.01±1.18 &
  40.18±0.55 &
  86.81±2.24 &
  84.15±3.05 &
  76.89±0.97 &
  76.52±0.89 &
  86.14±0.25 &
  60.92±0.07 \\
GCN &
  60.81±2.95 &
  45.87±0.88 &
  33.26±1.15 &
  76.97±3.97 &
  65.78±4.16 &
  87.18±1.12 &
  79.85±0.78 &
  86.79±0.31 &
  62.18±0.26 \\
\midrule
GCNII &
  63.44±0.85 &
  41.96±1.02 &
  36.89±0.95 &
  80.46±5.91 &
  84.26±2.13 &
  {\ul 88.46±0.82} &
  {\ul 79.97±0.65} &
  89.94±0.31 &
  63.39±0.61 \\
$\text{H}_2$GCN &
  52.30±0.48 &
  30.39±1.22 &
  {\ul 38.85±1.17} &
  {\ul 85.90±3.53} &
  {\ul 86.23±4.71} &
  87.52±0.61 &
  {\ul 79.97±0.69} &
  87.78±0.28 &
  (M) \\
MixHop &
  36.28±10.22 &
  24.55±2.60 &
  33.13±2.40 &
  76.39±7.66 &
  60.33±28.53 &
  65.65±11.31 &
  49.52±13.35 &
  87.04±4.10 &
  {\ul 65.64±0.27}\\
GCN+JK &
  64.68±2.85 &
  {\ul 53.40±1.90} &
  32.72±2.62 &
  80.66±1.91 &
  66.56±13.82 &
  86.90±1.51 &
  73.77±1.85 &
  {\ul 90.09±0.68} &
  {63.45±0.22} \\
  \midrule
ClenshawGCN &
  \textbf{69.44±2.06} & 
  \textbf{62.14±1.65} &
  \textbf{42.08±1.99} &
  \textbf{93.36±2.35} &
  \textbf{92.46±3.72} &
  \textbf{88.90±1.26} &
  \textbf{80.34±1.26} &
  \textbf{91.99±0.41} & 
  \textbf{66.26±0.27} \\ \bottomrule
\end{tabular}%
\label{tb:full_residual}
}
\end{table*}

\begin{table*}[htp]
\centering

\small
\caption{
Comparison with spectral models.
Mean classification accuracies {(±95\% confidence intervals)}
on 20 random 60\%/20\%/20\% train/validation/test splits are displayed. 
Besides the ClenshawGCN, 
all the results are taken directly from \cite{He2021bern}.
}

\resizebox{1.98\columnwidth}{!}{%
\begin{tabular}{@{}lllllllll@{}}
\toprule
  Datasets &Chameleon &Squirrel &Actor &Texas &Cornell & Cora & Citeseer & Pubmed \\
  $|\mathcal{V}|$ &2,277 & 5,201 & 7,600 & 183  & 183 & 2,708 & 3,327 & 19,717 \\
  \midrule

ChebNet &
  59.51±1.25 &
  40.81±0.42 &
  37.42±0.58 &
  86.28±2.62 &
  83.91±2.17 &
  87.32±0.92 &
  79.33±0.57 &
  87.82±0.24 \\
ARMA &
  60.21±1.00 &
  36.27±0.62 &
  37.67±0.54 &
  83.97±3.77 &
  85.62±2.13 &
  87.13±0.80 &
  80.04±0.55 &
  86.93±0.24 \\
APPNP &
  52.15±1.79 &
  35.71±0.78 &
  39.76±0.49 &
  90.64±1.70 &
  91.52±1.81 &
  88.16±0.74 &
  80.47±0.73 &
  88.13±0.33 \\
GPRGNN &
  67.49±1.38 &
  50.43±1.89 &
  39.91±0.62 &
  92.91±1.32 &
  91.57±1.96 &
  88.54±0.67 &
  80.13±0.84 &
  88.46±0.31 \\
BernNet &
  68.53±1.68 &
  51.39±0.92 &
  41.71±1.12 &
  92.62±1.37 &
  92.13±1.64 &
  88.51±0.92 &
  80.08±0.75 &
  88.51±0.39 \\
ChebNetll &
  \textbf{71.37±1.01} &
  {\ul 57.72±0.59} &
  {\ul 41.75±1.07} &
  {\ul 93.28±1.47} &
  {\ul 92.30±1.48} &
  {\ul 88.71±0.93} &
  \textbf{80.53±0.79} &
  {\ul 88.93±0.29} \\
  \midrule
ClenshawGCN &
  {\ul 69.44±0.92} & 
  \textbf{62.14±0.70} &
  \textbf{42.08±0.86} &
  \textbf{93.36±0.99} &
  \textbf{92.46±1.64} &
  \textbf{88.90±0.59} &
  {\ul 80.34±0.57} &
  \textbf{91.99±0.17} \\ 
  \bottomrule
\end{tabular}%
\label{tb:full_many}
}
\end{table*}
\label{sec:compare_residual}
In this subsection, we illustrate the effectiveness of ClenshawGCN's residual connections 
by comparing ClenshawGCN with other spatial models, including
GCNII~\cite{chen2020gcnii}, $\text{H}_2$GCN~\cite{zhu2020beyond}, 
MixHop~\cite{abu2019mixhop} and JKNet~\cite{keyulu2018jknet}. 
Among them, GCNII is equipped with initial residual connections, 
and the others are equipped with dense residual connections. 
Moreover, the way of combining multi-scale representations is more complex than weighted sum in $\text{H}_2$GCN and JKNet. 
As shown in Table~\ref{tb:full_residual}, our ClenshawGCN outperforms all the baselines. 

On one hand, in line with our expectations, 
ClenshawGCN outperforms all the baselines on the {\ul heterophilic datasets} by a significantly large margin 
including {\ul $\text{H}_2$GCN}, which is tailored for heterophilic graphs. 
This illustrates the effectiveness of borrowing spectral characteristics. 

On the other hand, 
ClenshawGCN even shows an advantage over {\ul homophilic datasets}, 
though the compared spatial models, such as GCNII and JKNet, 
are strong baselines on such datasets. 
Especially, for the PubMed dataset, 
ClenshawGCN achieves state of art.

\subsection{Comparing ClenshawGCN with Spectral Baselines} 
In Section~\ref{sec:neg_residue}, 
we have proved that ClenshawGCN acts as a spectral model
and simulates any $K$-order polynomial filter based on
$\left\{U_{\ell}  \right\}_{\ell=0}^{\ell=K}$.
In this subsection, we compare ClenshawGCN with strong spectral GNNs, 
including 
ChebNet~\cite{Defferrard2016cheb}, 
APPNP~\cite{klicpera2019appnp}, 
ARMA~\cite{bianchi2021arma}, 
GPRGNN~\cite{chien2021gprgnn}, 
BernNet~\cite{He2021bern}, 
and ChebNetII~\cite{he2022chebii}. 
Among them, APPNP simulates polynomial filters with \textit{fixed} parameters, 
ARMA GNN simulates ARMA filters~\cite{narang2013signalarma}, 
and the rest simulate \textit{learnable} polynomial filters based on the Chebyshev basis, Monomial basis or Bernstein basis. 


As reported in Table~\ref{tb:full_many}, 
ClenshawGCN outperforms almost all the baselines on each dataset, except for Chameleon and Citeseer. 
Specifically, ClenshawGCN outperforms other models on the Squirrel dataset 
by a large margin of $7.66\%$. 

Note the comparison between ClenshawGCN and {ChebNetII}. 
ChebNetII gains \textit{extra power} from 
the leveraging of Chebyshev nodes, 
which is crucial for polynomial interpolation.   
However, without the help of Chebyshev nodes,  
ClenshawGCN is comparable to ChebNetII in performance. 
The extra power of ClenshawGCN may come from the entangled non-linear transformations.

\subsection{Ablation Analysis}

\begin{figure}[htp]
    \includegraphics[width=\columnwidth]{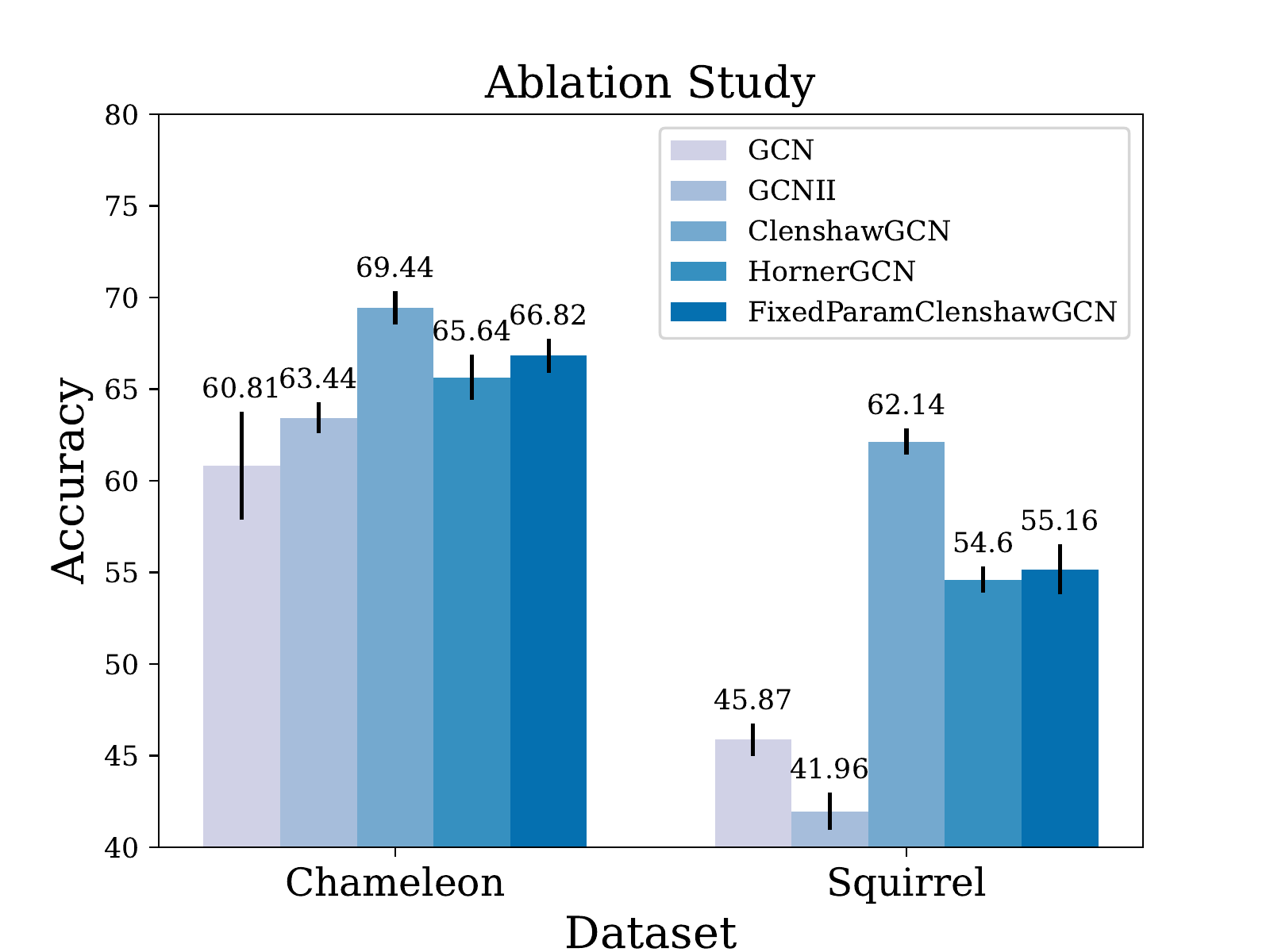}
    \caption{Results of the ablation study. HornerGCN and FixedParamClenshawGCN are weakened versions of ClenshawGCN. 
    HornerGCN is only equipped with adaptive initial residue, 
    and FixedParamClenshawGCN is only equipped with negative second-order residue. 
    The performance of these two ablation models are worse than a complete ClenshawGCN 
    but outperform GCN and GCNII.}
    \label{fig:ablation}
    \Description{figure description}
\end{figure}

For ClenshawGCN, the core of the design is the two residual connection modules. 
In this section, 
we conduct ablation analyses on these two modules to verify their contribution. 

\subsubsection*{\underline{Ablation Model: HornerGCN}}
We use HornerGCN as an ablation model to verify 
the contribution of \textit{negative residues}.
Recal Theorem~\ref{thm:horner}, 
the corresponding polynomial filter of HornerGCN is 
$$
    h_{Horner}(\mu) = \sum_{\ell=0}^{K}{\alpha}_{K-\ell} \mu^{\ell}, 
$$
which uses the Monomial basis. While the complete form of 
our ClenshawGCN borrows the use of chebyshev basis by negative residues.

\subsubsection*{\underline{Ablation Model: FixedParamClenshawGCN}}
In the FixedParamClenshawGCN model, 
we verify the contribution of \textit{adaptive} initial residue 
by fixing 
$\vec{\alpha} = \left[\hat{\alpha}_0, \hat{\alpha}_1, \cdots, \hat{\alpha}_K \right]$ with 
$$
    \hat{\alpha}_{\ell} = {\alpha}(1-{\alpha})^{K-\ell}, \quad \hat{\alpha}_{0} =(1-\alpha)^K  
$$
following APPNP~\cite{klicpera2019appnp}, 
where 
$ \alpha \in \left[ 0,1 \right] $ 
is a hyperparameter. 
The corresponding polynomial filter of FixedParamClenshawGCN is: 
$$
    h_{Fix}(\mu) = \sum_{\ell=0}^{K}{\hat{\alpha}}_{K-\ell} U_{\ell}(\mu).
$$

\subsubsection*{\underline{Analysis}}
We compare the performance of HornerGCN and FixedParamClenshawGCN with 
GCN, GCNII and ClenshawGCN on two median-sized datasets: Chameleon and Squirrel. 
As shown in Figure~\ref{fig:ablation}, either removing the negative second-order residue 
or fixing the initial residue causes an obvious drop in Test Accuracy. 

Noticeably, 
with only one residual module, 
HornerGCN and FixedParamClenshawGCN still outperform 
homophilic models such as GCNII. 
For HornerGCN, the reason is obvious: HornerGCN simulates any polynomial filter, 
which is favorable for heterophilic graphs.   
For FixedParamClenshawGCN, though the coefficients of $U_{\ell}(\mu)$ are fixed, 
the contribution of each fused level (\textit{i.e.} $\mathbf{\tilde{P}^{\ell}}\mathbf{H}^{*}$) 
is no longer definitely to be \textit{positive} as in GCNII since each chebyshev polynomial 
consists of terms with \textit{alternating signs}, \textit{e.g.} $ U_4(\mathbf{\tilde{P}}) =  
16 \mathbf{\tilde{P}}^4 \textcolor{red}{-} 12 \mathbf{\tilde{P}}^2 \textcolor{red}{+} 1 $,  which  
breaks the underlying homophily assumption.

\section{Conclusion}
In this paper, we propose 
ClenshawGCN, a GNN model equipped with a 
novel and neat residual connection module
that is able to mimic a spectral polynomial filter. 
When generating node representations for the next layer(\textit{i.e.} $\mathbf{H}^{(\ell+1)}$),
our model should only connect to the initial layer(\textit{i.e.} $\mathbf{H}^{(0)}$) \textit{adaptively},
and to the {second last layer}(\textit{i.e.} $\mathbf{H}^{(\ell-1)}$) \textit{negatively}.
The construction of 
this residual connection inherently uses {Clenshaw Summation Algorithm}, 
a numerical evaluation algorithm for calculating weighted sums of Chebyshev basis. 
We prove that 
our model 
implicitly simulates \textit{any} polynomial filter based on the \textit{second-kind Chebyshev basis} entangle with non-linear layers, bringing it at least comparable expressive power with state-of-art polynomial spectral GNNs.  
Experiments demonstrate our model's  superiority 
either compared with other spatial models with residual connections or with spectral models. 

For future work, 
a promising direction is to further investigate the 
mechanism and potential of such 
spectrally-inspired models entangled with non-linearity, 
which seems be able to incorporate the strengths of both sides.

\bibliographystyle{ACM-Reference-Format}
\bibliography{references}

\end{document}